\documentclass[letterpaper, 10 pt, conference]{ieeeconf}  %
\IEEEoverridecommandlockouts                              %
\usepackage{booktabs}

\usepackage{amsthm}
\usepackage{amsmath}
\usepackage{amsfonts}
\usepackage{amssymb}
\usepackage{graphicx}
\usepackage{mathtools}
\usepackage{commath}
\usepackage{arydshln}  %
\usepackage{calrsfs}  %
\usepackage{xcolor} %

\if CLASSOPTION compsoc
	\usepackage[caption=false, font=normalsize, labelfont=sf, textfont=sf]{subfig}
\else
	\usepackage[caption=false, font=footnotesize]{subfig}
\fi
\usepackage{listings}
\usepackage{float}
\usepackage{url}
\usepackage{algorithm}
\usepackage{algorithmicx}
\usepackage[noend]{algpseudocode}

\usepackage{blindtext}
\usepackage{booktabs} %

\DeclareMathOperator*{\argmin}{arg\,min}

\theoremstyle{definition}
\newtheorem{theorem}{Theorem}[section]
\newtheorem{lemma}[theorem]{Lemma}

\theoremstyle{definition}
\newtheorem{prob}{Main Problem}[section]

\newtheorem{defn}{Definition}[section]

\newcommand\oprocendsymbol{\hbox{$\bullet$}}
\newcommand\oprocend{\relax\ifmmode\else\unskip\hfill\fi\oprocendsymbol}

\newtheorem*{remark}{Remark}

\usepackage{alphalph} %

\newcommand{\ie}{i.e., }

\title{\LARGE \bf
Robotic Coverage for Continuous Mapping Ahead of a Moving Vehicle
}

\urldef{\smith}\url{stephen.smith@uwaterloo.ca}
\urldef{\gilhuly}\url{barry.gilhuly@uwaterloo.ca}

\author{Barry Gilhuly and Stephen L.\ Smith
\thanks{This research is partially supported by the Natural Sciences and Engineering Research Council of Canada (NSERC) and by General Dynamics Land Systems - Canada. }
\thanks{The authors are with the Department of Electrical and Computer Engineering, University of Waterloo, Waterloo ON, N2L 3G1 Canada (\gilhuly, \smith)
    }
}

\begin{document}

\maketitle
\thispagestyle{empty}

\begin{abstract}

In this paper we investigate the problem of using a UAV to provide current map information of the environment in front of a moving ground vehicle.  We propose a simple coverage plan called a conformal lawn mower plan that enables a UAV to scan the route ahead of the ground vehicle.  The plan requires only limited knowledge of the ground vehicle's future path.  For a class of curvature-constrained ground vehicle paths, we show that the proposed plan requires a UAV velocity that is no more than twice the velocity required to cover the optimal plan.  We also establish necessary and sufficient UAV velocities, relative to the ground vehicle velocity, required to successfully cover any path in the curvature restricted set.  In simulation, we validate the proposed plan, showing that the required velocity to provide coverage is strongly related to the curvature of the ground vehicle's path.  Our results also illustrate the relationship between mapping requirements and the relative velocities of the UAV and ground vehicle.

\end{abstract}

\section{Introduction}

Cooperative mapping using a ground vehicle and one or more unmanned aerial vehicles (UAVs) is an extremely active area of research, with applications in search and rescue, agriculture, military reconnaissance, mapping and inspection~\cite{waslander2013unmanned}.  Working together, an autonomous UAV and a ground vehicle can explore~\cite{cantelli2013autonomous}, inspect~\cite{kazmi2011adaptive}, monitor~\cite{grocholsky2006cooperative}, and track~\cite{yu2015cooperative}.  Such collaborative applications take advantage of the strengths of both vehicles:  the UAV provides a higher vantage point, wider field of view, and faster movement~\cite{li2016hybrid, saska2014coordination, cantelli2013autonomous}, while the ground vehicle gives more detailed imagery, improved location management and can carry supplies to extend the range of the UAV~\cite{mathew2015multirobot}.

In the mapping domain, several studies investigate using a collaborative team of UAVs and ground vehicles to explore an unknown terrain~\cite{kim2014multi,cantelli2013autonomous,hood2017bird}. However these studies generally employ a UAV in a stationary “eye in the sky” position above the ground vehicle~\cite{hood2017bird}, creating a high vantage point, but providing only a limited view of the area that lies ahead.  Others have investigated using the UAV's faster velocity and easier navigation to map a region quickly, allowing a ground vehicle to plan a safe route through difficult terrain while visiting locations of interest~\cite{lazna2018cooperation, cantelli2013autonomous};  in these studies though, the region to be mapped is fixed, and not restricted by the motion or capabilities of the ground vehicle.  In another example, machine learning techniques~\cite{christie2017radiation} are used to plan the ground vehicle's route, but not predict it.  Still other mapping studies use UAVs to explore points of interest while the ground vehicle acts simply as a mobile supply depot, providing support and resources to keep the UAVs flying~\cite{ren2018path, mathew2015multirobot}.

When planning coverage paths, policies typically use a Boustrophedon path, otherwise known as a lawn mower pattern~\cite{choset2000coverage,galceran2013survey}.   Used where the environment is known, these policies first decompose the space into convex polygons, then build a plan to provide complete coverage.  Building on this approach, other studies such as \cite{bochkarev2016minimizing} find the shortest coverage path by minimizing the number of turns on the path. 

A related topic is persistent monitoring, where the environment is constantly changing within a fixed area.  Agents must repeatedly cover terrain that is changing over time~\cite{ahmadi2005continuous,smith2012persistent}.  Unlike persistent monitoring, the mappable area in our scenario has an effective expiry time -- the UAV must cover locations before the ground vehicle comes within a specified ``lookahead distance''.  This is further complicated in that the rate of expiry is not consistent, changing depending on the curvature of the ground vehicle's path.

In this paper we consider a path planning problem where a moving ground vehicle is dependent upon a UAV for advance information regarding the upcoming route.  As the ground vehicle travels, no terrain is allowed to come within a specified distance without first being mapped/covered by the UAV.   The UAV has access to a limited window of the ground vehicle's upcoming path, and must build and execute an appropriate coverage plan based on that knowledge.  Due to the short-term nature of the path information, the UAV must continuously update its mapping plan as the ground vehicle advances and supplies updated directions.  We adapt the common lawn mower plan to take advantage of the limited information, and then characterize the efficiency of our approach by establishing lower and upper bounds on the UAV's velocity. 
	
\paragraph*{Contributions}
The primary contributions of this paper are threefold.  First, we introduce the problem of providing continuous coverage of the path ahead of a moving ground vehicle.  Second, we present a plan capable of providing coverage with only limited knowledge.  We further establish upper and lower bounds on the length of this plan and, based on that distance, estimate the required UAV velocities.  Third, we prove that when the curvature of the ground vehicle path is limited, our plan provides a UAV path that requires no more than twice the velocity of that required to cover the optimal path.  

\paragraph*{Organization}
The paper is structured as follows. In Section~\ref{section-problem-statement} we formally introduce the coverage problem and the assumptions on the ground and aerial vehicles.  In Section~\ref{section-oracle-coverage} we present the Conformal Lawn Mower coverage plan.  In Section~\ref{sec:coverage-efficiency}, the efficiency of the algorithm is developed and we present the solution to the main problem posed in this paper.  In Section~\ref{section-simulation-results} we discuss our simulation results.  Finally, Section  \ref{section-conclusion} contains concluding remarks and notes on possible future work.

\section{Problem Statement}  \label{section-problem-statement}

	\begin{figure} 
		\centering
		\includegraphics[width=0.65\linewidth, clip=true]{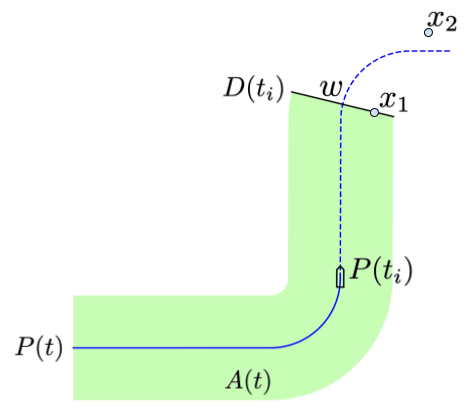}
		\caption{Ground Vehicle Path and the Coverage Corridor.}
		\label{path-illustration}
	\end{figure}
	
	Consider a ground vehicle moving through an environment in $\mathbb{R}^2$ following a smooth path $P(t)$ at a constant velocity $v_{\text{gv}}$ for $t\in [0,t_{\max}]$.	
		
	A distance $d_{\text{map}}$ ahead of the ground vehicle, defines the length of the coverage area. The width of the coverage area $w$ is specified by the operator as a path parameter.   As the ground vehicle moves along the path, this $d_{\text{map}} \times w$ coverage area moves ahead of it, generating a mapping demand.   We generally expect that the total length of the path to be much greater than $w$.  
	
	The leading edge of this coverage area is the deadline.  If we define $t_{\text{map}}$ as the time required for the ground vehicle to traverse the map distance, $\frac{d_{\text{map}}}{v_\text{gv}}$, then given a unit normal, $\overrightarrow{n}$, at point $P(t + t_{\text{map}})$, we can define the leading edge of the coverage area, $D(t)$, as
	\begin{equation*}
	    D(t) = \left\{ x \in \mathbb{R}^2 |  x = P(t+ t_\text{map}) + \alpha \vec{n}, \; \alpha \in \big[-\frac{w}{2},\frac{w}{2}\big]\right\}.
	\end{equation*}
	
	We assume all ground vehicle paths, $P(t)$, are members of the set
	\begin{align*}
	\mathbb{P} = \left[ P(t) \bigg \rvert c(t) \leq \frac{2}{w}, t \in [0,t_{\max}] \right],
	\end{align*}
	where $c(t)$ is the curvature of $P(t)$ at time $t$.  This curvature constraint ensures that as the vehicle progresses along the path P(t), the endpoints of the deadline D(t) always make non-negative progress along the boundary of the coverage area.
	
	From this we can more formally define coverage area $A(t)$ to be the union of points found by sweeping $D(t_\text{dl})$ along $P(t_\text{dl})$ for $t_\text{dl} \in [0,t+t_{\text{map}}]$, expressed as
	\begin{align*}
	A(t) = \cup_{t_\text{dl}=0}^{t+t_{\text{map}}} D(t_\text{dl}).
	\end{align*}
		
	A UAV is deployed to provide mapping imagery, using a monocular vision system to capture terrain data. Similar to~\cite{xu2011consensus}, we model the UAV motion using single integrator dynamics and focus on the high level planning problem. The UAV's camera has a fixed-size square optical footprint with sides of length $f$, where $f < w$.  If $f \geq w$, then the solution is to simply fly the UAV along the ground vehicle path at the same velocity, $v_{\text{uav}} = v_\text{gv}$.  The total area of the map covered by the UAV over the interval $[0,t]$ is denoted $M(t) \subset \mathbb{R}^2$.   
	
	 The UAV is unable to create an optimal mapping plan as it only has a limited window of the ground vehicle's upcoming path.

	Figure~\ref{path-illustration} shows an example path with the ground vehicle located at $P(t_i)$.   The coverage area starts at $P(0)$, is centered on $P(t)$, and continues to a point $d_{\text{map}}$ units ahead of the ground vehicle at $P(t+t_{\text{map}})$.  The environment is assumed to be free of obstacles that affect the UAV.  There may be obstacles that limit the possible trajectories of the ground vehicle; however, we assume that the UAV does not have access to this information.  
	
	 For all points in $x \in A(t)$,  the expiry time $t_{\text{exp}}(x)$ is defined as the time at which $x$ intersects with the deadline $D(t')$ for the first time $t' \leq t$.  If a point is not in $M(t)$ before expiring, then it is considered a coverage failure.  The expiry of a point $x$ is
	\begin{align*}
	t_\text{exp}(x) = \argmin_t \left\{x \in A(t) \right\}.
	\end{align*}
	For example, in Figure \ref{path-illustration}, the point $x_1$, seen on the line $D(t_1)$ has just expired. The point $x_2$ is still outside the coverage area.   

    At time $t = 0$, we assume the UAV is positioned at the beginning of its first pass on one side of the path, ready to start mapping.   The deadline is located at $P(0)$, with the ground vehicle not yet on the path.  After a delay of $\Delta t = \frac{f}{v_\text{gv}}$, enough time for the UAV to map the first pass of the path, the ground vehicle and the deadline begin to move forward. 

	Given this background information, the problem may be formally stated.
	
	\begin{prob}[Complete Coverage]\label{prb:complete-coverage}
		
		Consider a ground vehicle traveling through an environment following a path, $P(t) \in \mathbb{P}$, creating a coverage demand of $A(t)$.  A UAV travels ahead of the ground vehicle producing a coverage area of $M(t)$.  Assume the UAV has knowledge of an upcoming window of the ground vehicle's path, $P(\bar t), \bar t \in [t, t + \Delta t]$. Determine a plan for the UAV that guarantees
		\begin{equation}
		A(t) \subseteq M(t), \forall t \in [0, t_{\max}]. \label{eqn-high-level-description}
		\end{equation}
	\end{prob}
	We seek to characterize this plan's efficiency as follows. 
	\begin{prob}[Proof Of Efficiency]\label{prb:efficiency}
    	Given the plan determined by~\eqref{eqn-high-level-description}, what is the efficiency relative to the optimal coverage plan for the same path, $P(t)$?
    \end{prob}

\section{The Conformal Lawn Mower Path} \label{section-oracle-coverage}

It is well established that a simple, non-overlapping lawn-mower path is an optimal method for covering a rectangular area~\cite{choset2000coverage}.  We propose that for the ground vehicle path, $P(t)$, we can define a \emph{Conformal Lawn Mower} path such that the lines defining the back and forth motion of a regular lawn mower may no longer be parallel.  Instead, the angle between any two adjacent lines is allowed to range from parallel up to a maximum value defined by the curvature of the path and the UAV's optical footprint.  Refer to Figure~\ref{oracle-coverage-path-plan} where the UAV coverage plan (red dashed line) is overlaid on the ground vehicle path.

\begin{figure} 
	\centering
	\includegraphics[width=0.9\linewidth]{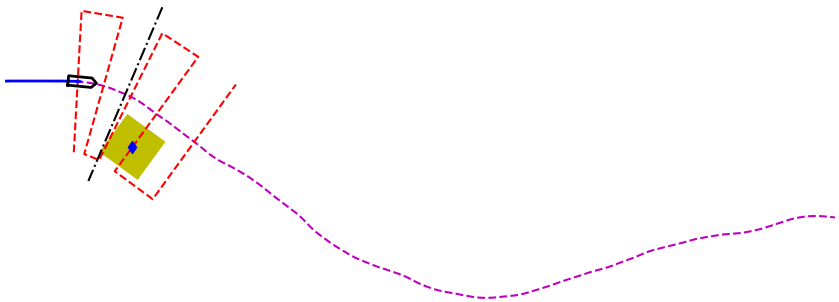}
	\caption{Conformal Lawn Mower plan with a limited window.}
	\label{oracle-coverage-path-plan}
\end{figure}

\begin{defn}[Traversal]
A \emph{Traversal} is a line segment of length $w$ perpendicular to and centered on the path.  To guarantee complete coverage, the distance between any two traversal lines has an upper bound of $f$.    
\end{defn}

\begin{defn}[Transit]
A \emph{transit} is defined as the section of the UAV's coverage plan that connects the ends of two adjacent traversals. Transits are assumed to follow the profile of the path edge (\ie an arc when the path is curved).
\end{defn}

\begin{defn}[Period]
A \emph{period} for a conformal lawn mower plan is a grouping of the movements required to cover a section of the path and return to the same position, but shifted forward along the path.  A period consists of the following movements: traversal, transit, traversal, transit.
\end{defn}

The conformed plan is a sequence of alternating traversals and transits that allow the UAV to completely map $A(t)$.  The procedure for constructing a conformal lawn mower path is shown in Algorithm~\ref{oracle-coverage-algorithm}.

\begin{algorithm}
	\caption{Conformal Lawn Mower Plan}\label{oracle-coverage-algorithm}
  \begin{enumerate}
		\item Add an initial traversal at $P(0)$ to the plan.
		\item Find the first point on the path $P(t)$ such that a traversal centred at $P(t)$ has an endpoint at distance $f$ from its corresponding endpoint on the last traversal.  \label{enm:walking-step}
		\item Add a transit to this traversal at $P(t)$ and the traversal to the plan, where successive transits alternate sides.
		\item If the ground vehicle has stopped, add a final transit and traversal to the plan and exit.
		\item Otherwise, when there is new path information, repeat from step~\ref{enm:walking-step}.
\end{enumerate}
\end{algorithm}

The UAV uses the provided path information to map the initially known $A(t)$ following the conformal plan.  As the ground vehicle moves forward and additional path information comes available, the UAV plan is extended, allowing the UAV to map the new territory.

\begin{theorem}[Complete Coverage] \label{thm:complete-coverage}
The Conformal Lawn Mower plan in Algorithm~\ref{oracle-coverage-algorithm} provides complete coverage of path $P(t)$.
\end{theorem}

\begin{proof}
From Algorithm~\ref{oracle-coverage-algorithm}, the ground vehicle path is sampled, placing a new traversal where necessary to maintain the maximum separation.  We start by placing a traversal at $t=0$.  Then, travelling the path,  calculate the distance between the endpoints of the last traversal added to the path, and a prospective one at the current location of P. When the distance of either endpoint from the previous traversal is $f$, the algorithm places a transit, locating it  on the opposite side from the previous one, then places the prospective traversal. By enforcing the distance between traversals to be $f$, the algorithm ensures we have complete coverage. This process repeats until the end of the path has been reached.

No two traversals are ever separated by more than $f$, so that two sequential passes of the UAV, one on each traversal, captures all of the area of $A(t)$ between those traversals in $M(t)$. Since all of $P(t)$ is sampled by traversals, and $A(t)$  is defined by $P(t)$, then 
\begin{equation}
    A(t) \subseteq M(t).
\end{equation}
Therefore, the Conformal Lawn Mower completely covers the swept area, $A(t)$, defined by the path, $P(t)$.
\end{proof}

\section{Coverage Efficiency} \label{sec:coverage-efficiency}
 
We begin by proving the performance of the conformal lawn mower plan. Using these results, we present our solution to Problem~\ref{prb:complete-coverage}.  Finally, we demonstrate the suboptimality of the conformal plan, by presenting a handcrafted alternative.

\subsection{Proof of Efficiency}
The UAV has perfect knowledge of the ground vehicle's intended path for a limited window -- the UAV is given the ground vehicle's path for the range $[t_{i}, t_{i}+\Delta t]$.  The UAV must create a coverage plan that ensures all of $P(t), t \in [t_{i}, t_{i}+\Delta t]$ is covered prior to expiry.  

We will demonstrate a worst case scenario that minimizes the distance the ground vehicle travels relative to the UAV.  Based on this, we can establish a sufficient relative velocity for the UAV to successfully cover any ground vehicle path within the curvature constrained set $\mathbb{P}$.  

\begin{theorem}[Efficiency] \label{thm:twice-optimal}
For any path P(t) in the set $\mathbb{P}$, the conformal lawn mower plan has a length that is no more than two times the optimal coverage plan.
\end{theorem}

To prove this result we require a few preliminary lemmas.

\begin{lemma}[The Optimal Straight Path Ratio] \label{lma:opt-straight-path}
For a straight path $P(t)$ (Figure~\ref{fig:uav-optimal-straight}), the ratio of the distance travelled by the UAV to that of the ground vehicle is $\frac{w}{f}$.
\end{lemma}

\begin{proof}
\begin{figure}
    \centering
    \includegraphics[width=0.75\linewidth]{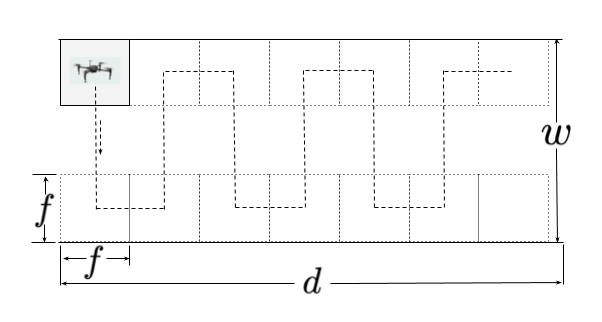}
    \caption{An optimal coverage plan for a straight path.}
    \label{fig:uav-optimal-straight}
\end{figure}
The ground vehicle travels down the centre of the path, moving a distance of $d$.  Therefore the total coverage demand is $w d$. The ratio of the ground vehicle velocity to the UAV is determined by
\begin{align*}
    \frac{d}{v_{\text{gv}}} \geq \frac{w d}{f v_{\text{uav}}}.
\end{align*}
Since the velocity of both vehicles is fixed, we can eliminate the time component on both sides and state this in terms of distance.  Therefore, the ratio of the vehicle distances is
\begin{align}
    \frac{d_{\text{uav}}}{d_{\text{gv}}} \geq 
     \frac{v_{\text{uav}}} {v_{\text{gv}}} 
    \geq \frac{w}{f}. \label{eqn:optimal-distance-ratio}
\end{align}

\end{proof}

\begin{remark}[Optimality of the Lawn Mower Coverage Plan]
    The lawn mower coverage plan, illustrated in Figure~\ref{fig:uav-optimal-straight}, is an optimal plan for the straight path.  Each traversal is $w-f$ in length and spaced $f$ apart, meaning that for one complete period of two traversals and two movements of $f$, the ratio of UAV distance to ground vehicle distances is
    \begin{align*}
        \frac{d_{\text{uav}}}{d_{\text{gv}}} \geq \frac{ 2 (w-f) + 2 f }{ 2 f } = \frac{w}{f}.
    \end{align*}
    \oprocend
\end{remark}

The length of the UAV coverage plan for any arbitrary path can not be any shorter than the optimal coverage plan for the straight path.  

\begin{lemma}[Arbitrary Paths have the Same Area] \label{lma:same-area}
An arbitrary path in $\mathbb{P}$ of length $d$ has an optimal coverage plan at least as long as the optimal coverage plan for a straight path of the same length. 
\end{lemma}

\begin{proof}
    We first prove an arbitrary path in $\mathbb{P}$ has the same area as an equivalent straight path with the same centre-line length.  The comparison of optimal coverage path lengths flows directly from this fact.  
    
    Let $S$ be an arbitrary path in $\mathbb{P}$ of length $d$ and width $w$.  The arbitrary path can be decomposed into a set of $n$ curve sections, $\left\{ s_1, s_2, \ldots, s_n \right\}$, where each section has a centre-line length $\Delta d$ such that
    \begin{align*}
        d = \sum_{i=1}^{n} \Delta d.
    \end{align*}
    We then approximate each segment $s_i$ by a segment $s_i'$, which has length $\Delta d$ and a constant curvature equal to the maximum curvature of $s_i$.  The concatenation of these segments $s_1',\ldots,s_n'$ creates a curve $S_n$.  Notice that by the smoothness of paths in $\mathbb{P}$, we have $S_n \to S$ as $n\to \infty$ and thus $\Delta d \to 0$.
    
    The total area of $S_n$ is the sum of the areas of all $n$ of its sections,  
    \begin{align*}
        \text{area}(S_n) = \sum_{i = 1}^n \text{area}(s_i'). 
    \end{align*}
    
    For each section, $s_i'$, the area is calculated in one of two ways.  If the section $s_i$ is straight, its area is $\Delta d w$.  Otherwise, letting $r$ be one over the curvature of the section, the area of curved section $s_i'$ is
    \begin{align*}
        \text{area}(s_i') &= \frac{\theta}{2\pi}\left((\pi(r + \frac{w}{2})^2 - (\pi(r + \frac{w}{2})^2\right) = \theta\left(rw\right).
    \end{align*}
    But we also know that $\theta = \frac{\Delta d}{r}$.  Substituting into our equation gives us
    \begin{align*}
        \text{area}(s_i') = \frac{\Delta d}{r} rw = \Delta d w.
    \end{align*}
    Therefore the area of $S$ is
    \begin{align*}
        \text{area}(S) &= \lim_{\Delta d \rightarrow 0} \sum_{i = 1}^n \Delta d w = d w.
    \end{align*}
    This is exactly the area of a straight path of length $d$ and width $w$.      
    
    From Lemma~\ref{lma:opt-straight-path}, the UAV must travel at least $w/f$ times as far as the ground vehicle when covering a straight path.  Since the arbitrary path has exactly the same area as the straight path, it must generate exactly the same coverage demand.  The UAV's ability to satisfy the coverage demand remains the same, governed by the size of its optical footprint, $f$.  Therefore, the optimal coverage plan for the arbitrary path in $\mathbb{P}$ must be at least as long as the optimal coverage plan for the equivalent straight path.
\end{proof}

\begin{remark}[The need for a curvature constraint]
Our analysis is restricted to paths in $\mathbb{P}$ whose curvature is at most $2/w$. If a path contains a curve with curvature greater than $2/w$, the deadline endpoint on the inside of the curve moves in the opposite direction of the ground vehicle motion, resulting in a reduced swept area $A(t)$.  In this scenario Lemma \ref{lma:opt-straight-path} no longer holds, and thus the analysis does not follow through.
\oprocend
\end{remark}

Based on Lemma \ref{lma:opt-straight-path} and Lemma \ref{lma:same-area}, we can now prove Theorem~\ref{thm:twice-optimal}. 

\begin{proof}[Proof of Theorem \ref{thm:twice-optimal}]
Consider a straight ground vehicle path of width $w$ with a UAV providing mapping coverage using an optical footprint of size $f$.  If we consider the UAV path as a series of traversals and transits, that path is maximized if the transits are all of length $f$, as illustrated in Figure~\ref{fig:max-path}. Since the maximum separation between traversals is fixed at $f$, to find the worst case distance ratio between the ground vehicle and the UAV, we must minimize the ground vehicle distance.

\begin{figure}
    \centering
    \includegraphics[width=0.65\linewidth]{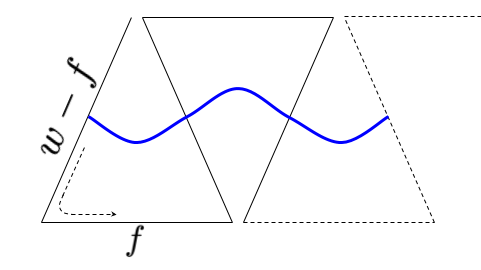}
    \caption{A path minimizing $d_{\text{gv}}$ with respect to $d_{\text{uav}}$.}
    \label{fig:max-path}
\end{figure}

  Starting with parallel traversals, we increase the angle between them. As the angle is increased, the curvature of the path increases, and the length of the path segment between the traversals decreases. Since each traversal must cross the path at right angles, the path must be a series of alternating circular arcs, with a curvature directly dictated by the angle between the traversals.  We can express the length of the ground vehicle's path segment between two traversals as
\begin{align}
    d_{\text{gv}} = r \theta = r \frac{f}{\frac{w}{2} + r}, \quad r \geq \frac{w}{2}. \label{eqn:gv-traversal-distance}
\end{align}
 Note that the distance in~\eqref{eqn:gv-traversal-distance} is minimized when $r = \frac{w}{2}$.
 
The UAV travels the length of one traversal, followed by a transit to the next traversal.  Therefore, the distance that the UAV must travel is
\begin{align}
    d_{\text{uav}} &= (w-f) + (\frac{w-f}{2} + r) \theta   \nonumber \\
    &=  (w-f) + (\frac{w-f}{2} + r) \frac{f}{\frac{w}{2} + r}  \nonumber \\
    &\leq (w-f) + f = w, \label{eqn:uav-max-distance}
\end{align}
since the traversals are separated by not more than $f$.  Therefore, the ratio of the UAV distance~\eqref{eqn:uav-max-distance} to the ground vehicle distance~\eqref{eqn:gv-traversal-distance} can be calculated
\begin{align}
    \frac{d_{\text{uav}}}{d_{\text{gv}}} &\leq \frac{w}{r \frac{f}{\frac{w}{2} + r}} \leq 2 \frac{w}{f}, \quad \text{ if } r = \frac{w}{2}. \label{eqn:worst-arbitrary}
\end{align}
The upper bound on the ratio of the velocity the UAV requires relative to the ground vehicle velocity on a arbitrary path is then given by
\begin{align} 
    \frac{v_{\text{uav}}}{v_{\text{gv}}} =
    \frac{d_{\text{uav}}}{d_{\text{gv}}}  &\leq 2 \frac{w}{f}. 
\end{align}

We know from Lemma~\ref{lma:same-area} and~\eqref{eqn:optimal-distance-ratio} that the ratio of the length of the coverage plan to the ground vehicle distance for any arbitrary path must be at least $\frac{w}{f}$.  Therefore for any arbitrary path in $\mathbb{P}$, following a conformal lawn mower plan can require no more than twice the velocity necessary for the optimal coverage plan on the same path.
\end{proof}

\begin{remark}[Another Optimal Path]
     Note that a small modification to the worst case UAV plan shown in Figure~\ref{fig:max-path} results in the optimal plan.  In particular, by traveling each traversal in the opposite direction, the UAV travels in a 'W' motion on traversals, and the transit between each traversal is 0 instead of $f$.  The resulting path has length $\frac{w}{f}d_{\text{gv}}$, which is exactly the optimal solution. If UAV had global knowledge of the path, it could select the appropriate direction to perform the traversals.  However, given its limited knowledge of the future path, this is not possible, and the resulting path can be a factor of two times longer.  This illustrates the impact of limited path information on the coverage efficiency. 
    \oprocend
\end{remark}

\begin{remark}[A near worst-case path]
     A single curve of maximum curvature $2/w$ is in $\mathbb{P}$ and provides a ratio of distances traveled that is nearly as large as for the path in Figure~\ref{fig:max-path}.  In particular, given the UAV and ground vehicle distances 
      \begin{align*}
          d_{\text{gv}} = 2 r \frac{f}{r + \frac{w}{2} }, \quad d_{\text{uav}} &= 2r \frac{f}{r + \frac{w}{2} } + 2(w-f),
      \end{align*}
      it is straightforward to calculate their ratio as 
      \begin{align}
          \frac{d_{\text{uav}}}{d_{\text{gv}}} &= 1 + \frac{w-f}{f} \left(\frac{r + \frac{w}{2}}{r} \right) \nonumber \\
          &= 1 + \frac{2(w-f)}{f} \nonumber \\
           &\leq 2 \frac{w}{f}.  \label{eqn:near-worst-case}
      \end{align}
      The ratio of \eqref{eqn:near-worst-case} is at maximum when $r = \frac{w}{2}$ and therefore reduces to 
     \begin{align}
         \frac{d_{\text{uav}}}{d_{\text{gv}}} &= 1 + \frac{2(w-f)}{f} = 2 \frac{w}{f} - 1. \nonumber
    \end{align}
     \oprocend
 \end{remark}

\subsection{The Correctness of the Conformal Lawn Mower Plan}
With Theorem~\ref{thm:complete-coverage} and Theorem~\ref{thm:twice-optimal} we have shown both complete coverage of the path $P(t)$ and the sufficient velocity the UAV requires for any path in $\mathbb{P}$.  We can now state our main result.

\begin{theorem}[Correctness of Conformal Lawn Mower Plan] \label{thm:correctness-of-plan}
    Consider a ground vehicle, travelling at velocity $v_{\text{gv}}$, with initial condition at the start of path $P(t), t = 0$. The path $P(t)$ has a coverage width of $w$ and a maximum curvature $\frac{2}{w}$.  Then, a UAV, with velocity $\geq 2\frac{w}{f} v_{\text{gv}}$ and following the conformal lawn mower path solves problem \ref{prb:complete-coverage}.
\end{theorem}

\begin{proof}
For complete coverage of $P(t)$, the region $A(t)$ swept out by $D(t), t \in [0, t_{\max}]$, must be entirely within the mapped area $M(t)$ before expiry.  Based on Theorem~\ref{thm:complete-coverage}, we can assert that $M(t)$ contains all of $A(t)$.  We must now prove that no elements of $A(t)$ expired before they were included within $M(t)$.

From the initial conditions, the UAV starts ahead of the deadline with at least one completed traversal already mapped before the ground vehicle starts moving.  For all remaining elements of $A(t)$ to be mapped correctly, we only need to show that the UAV maintains or extends its position ahead of the ground vehicle.  If the UAV uses a velocity that is at least $\frac{2w}{f}$, where $w$ is the width of the path and $f$ the UAV's optical footprint, then Theorem~\eqref{thm:twice-optimal} asserts this is true.

Therefore, all of $A(t)$ is successfully mapped, and Problem \ref{prb:complete-coverage} is solved.
\end{proof}

\subsection{Suboptimality of the Conformal Lawn Mower Plan}

\begin{table} 
\begin{center}
\caption{A Comparison of Conformal vs. Handcrafted Plans.}
\label{tbl:conf-vs-opt}
\begin{tabular}{@{} c  c  c  c  c @{}}
\toprule
UAV & \multicolumn{2}{ c }{Conformal} & \multicolumn{2}{ c }{Handcrafted}  \\ 
Velocity & Distance & \%Coverage & Distance & \%Coverage \\
\midrule
20	&	8000	&	22	&	8000	&	65 \\
21	&	8400	&	23	&	8400	&	89\\
22	&	8800	&	28	&	8800	&	100\\
23	&	9200	&	35	&	9027	&	100\\
24	&	9600	&	43	&	9038	&	100\\
25	&	10000	&	53	&	9040	&	100\\
26	&	10400	&	70	&	9036	&	100\\
27	&	10800	&	89	&	9042	&	100\\
28	&	11200	&	99	&	9044	&	100\\
29	&	11588	&	100	&	9038	&	100\\
30	&	11781	&	100	&	9039	&	100\\
\bottomrule
\end{tabular}
\end{center}
\end{table}

We have shown that the conformal lawn mower plan is within a factor of two of the optimal plan. For some paths in $\mathbb{P}$ there may be more efficient coverage solution that minimizes the ratio  $\frac{v_{\text{uav}}}{v_{\text{gv}}}$.  With full path knowledge, a coverage plan may be proposed that reduces the scanning overlap, and requires a lower sufficient velocity from the UAV as a result. 
Consider a path $P(t)$ with coverage width $w = 400$ and a maximum curvature of $\frac{1}{200}$.  A handcrafted coverage plan for $P(t)$ is presented in Figure~\ref{fig:near-optimal-coverage}, with the equivalent conformal plan in Figure~\ref{fig:conformal-coverage}.  Simulation results of both plans are presented in Table~\ref{tbl:conf-vs-opt}.   From these results, the handcrafted path reduces the minimum required by 7~m/s. The handcrafted plan, while not necessarily optimal, is clearly much better.  

We note that finding the optimal path appears to be an NP hard problem.  According to~\cite{arkin2000approximation}, the lawn mowing problems are NP hard in general; however, whether our formulation with the additional constraints is NP hard is the subject of further investigation.  

\begin{figure}
    \centering
	\subfloat[Handcrafted]{
        \includegraphics[width=0.65\linewidth]{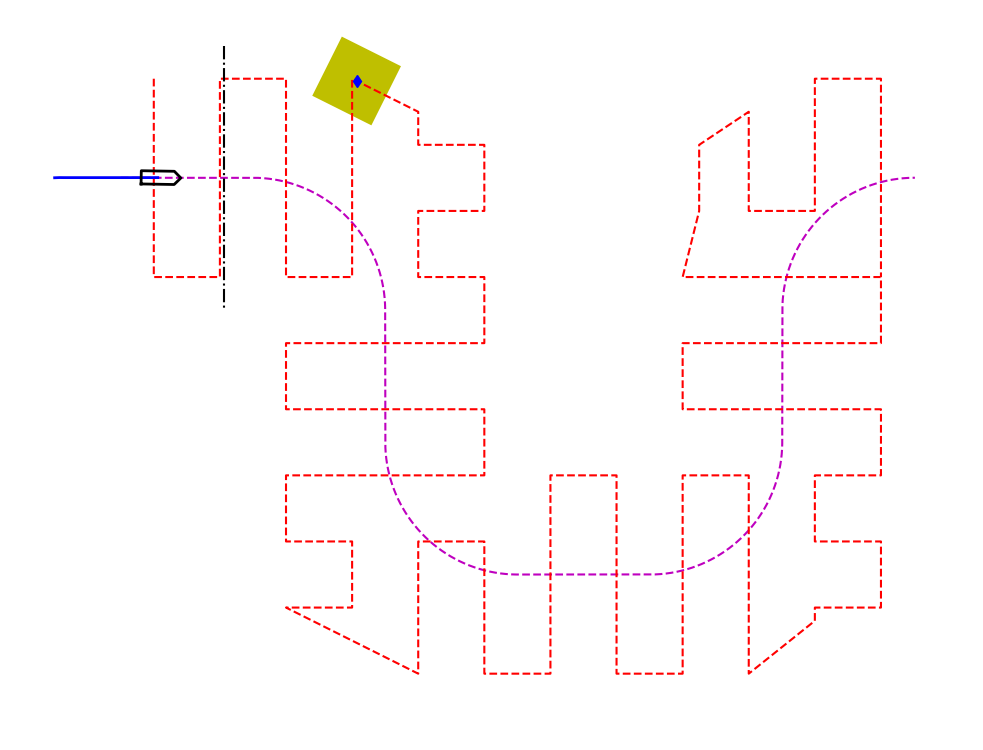}
        \label{fig:near-optimal-coverage}
    }    \\
	\subfloat[Conformal]{
        \includegraphics[width=0.65\linewidth]{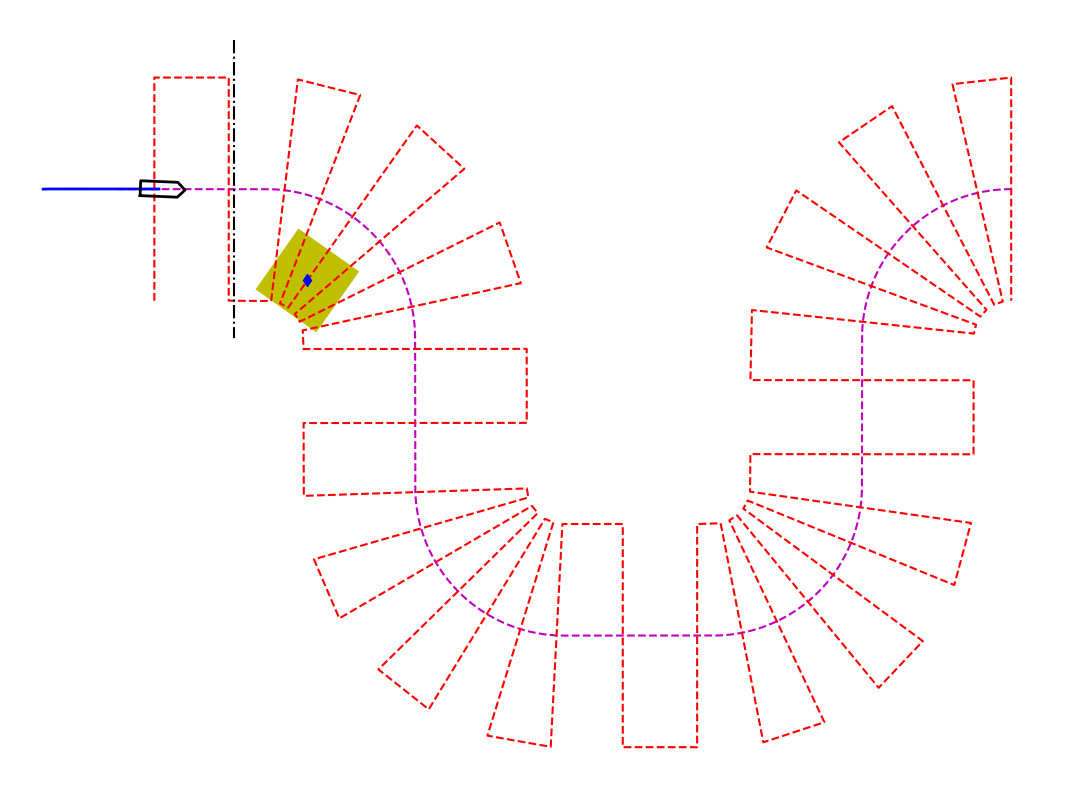}
        \label{fig:conformal-coverage}
    }    
    \caption{Two Coverage plans over a path with width 400m, curvature $\frac{1}{200}$.}
\end{figure}

\section{Simulation Results} \label{section-simulation-results}

We performed simulations in two scenarios, varying the UAV velocity on different types of ground vehicle paths (straight, decreasing curvature, and randomly generated), and using a single velocity while progressively decreasing the curvature.  For all simulations, the fixed parameters are: $v_{\text{gv}} = 5$ m/s, $w = 400$ m, $f = 100$ m.
In Table~\ref{tbl:coverage-at-100} the results of several simulations are shown.  The control case, a straight path, reaches full coverage between 20 and 21 m/s, as expected if we allow for slight rounding errors in the simulation.  Simulations were run for a random path with minimum curvature of $1/200$, as well as curvatures ranging from $1/200$ to $1/1000$.  The full simulation results are displayed in Figure~\ref{fig:coverage-and-speed}, while illustrations of some of the test paths can be seen in Figure~\ref{fig:increasing-speed-different-paths}. 
 \begin{figure}
     \centering
     \includegraphics[width=0.95\linewidth]{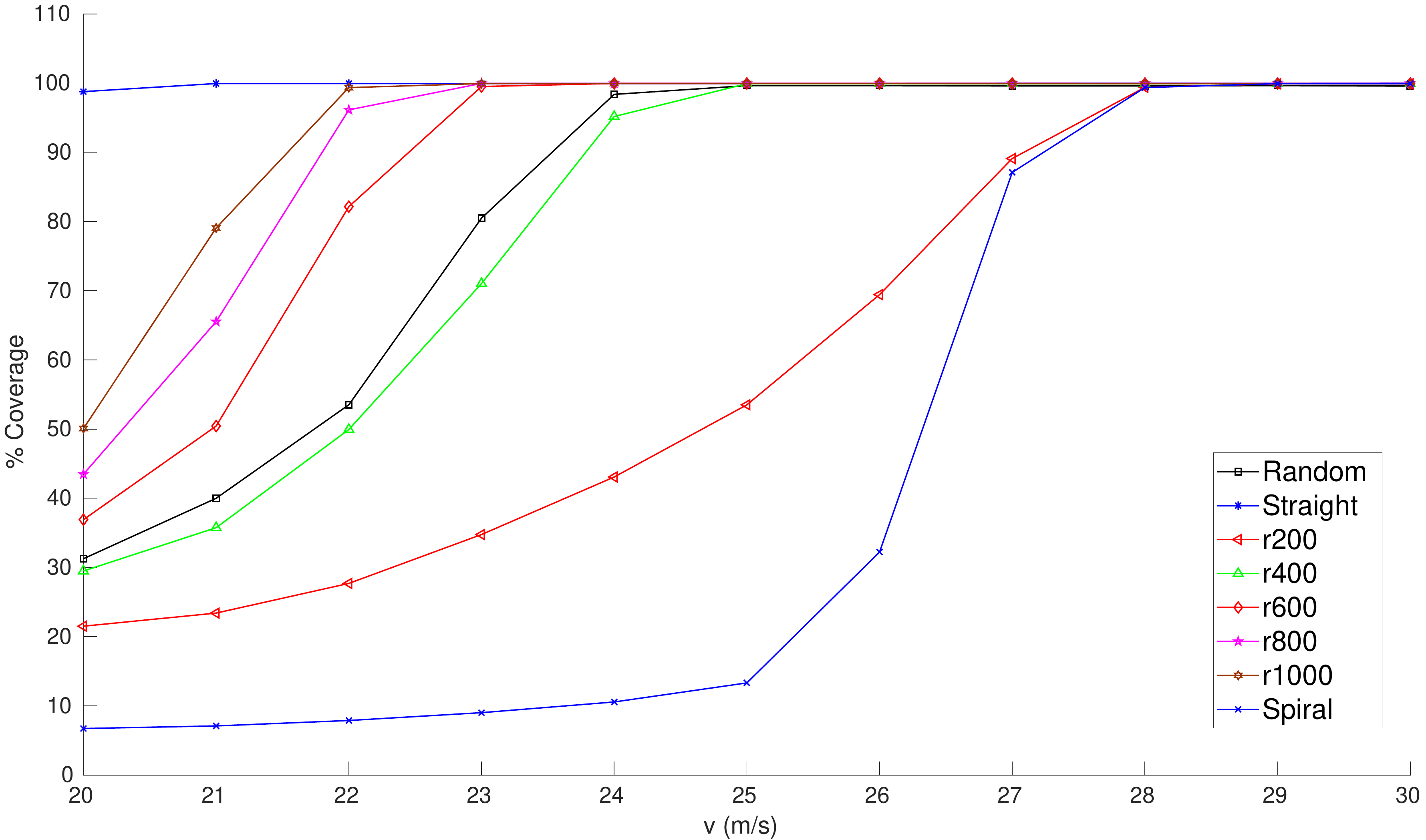}
     \caption{Percent coverage as a function of UAV velocity and path type.}
     \label{fig:coverage-and-speed}
 \end{figure}

\begin{figure*}
    \centering
    \subfloat[20 m/s]{
        \includegraphics[width=0.2\linewidth]{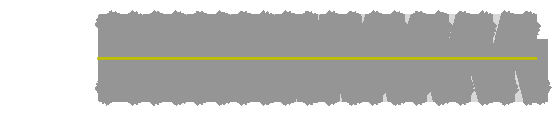}
    }
    \subfloat[22 m/s]{
        \includegraphics[width=0.2\linewidth]{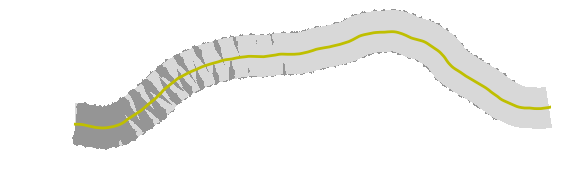}
    }
    \subfloat[22 m/s]{
        \includegraphics[width=0.2\linewidth]{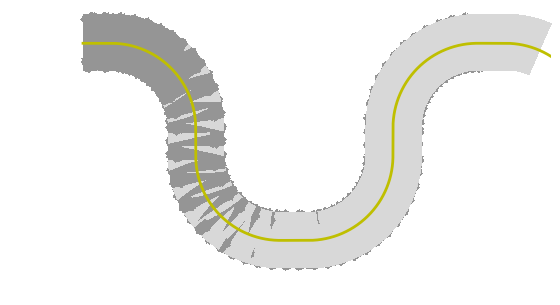}
    } 
    \subfloat[26 m/s]{
        \includegraphics[width=0.2\linewidth]{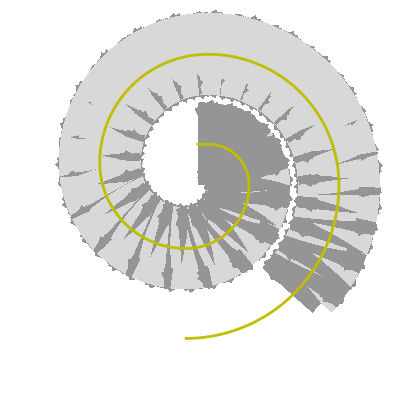}
    } \\

    \subfloat[21 m/s]{
        \includegraphics[width=0.2\linewidth]{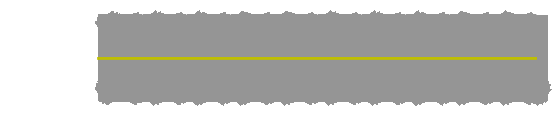}
    }
    \subfloat[25 m/s]{
        \includegraphics[width=0.2\linewidth]{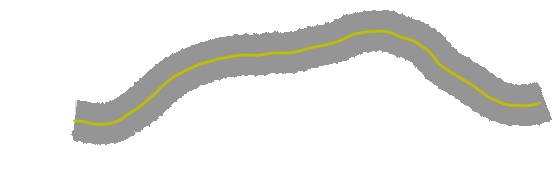}
    } 
    \subfloat[25 m/s]{
        \includegraphics[width=0.2\linewidth]{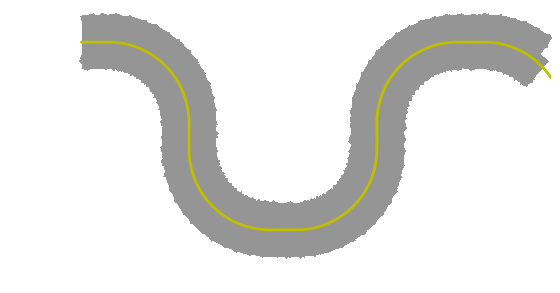}
    }  
        \subfloat[29 m/s]{
        \includegraphics[width=0.2\linewidth]{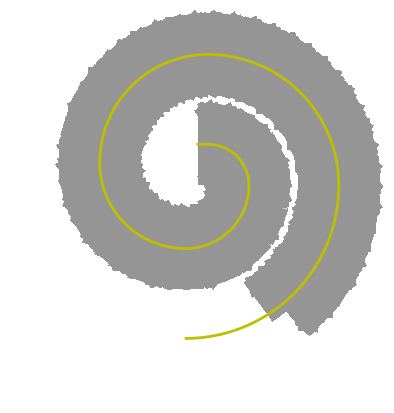}
    } \\

    \caption{Increasing UAV velocity over various path configurations.}
    \label{fig:increasing-speed-different-paths}
\end{figure*}
In all cases, the UAV and the ground vehicle start on the left side of the path, with the dark grey areas indicating successful mapping.  Areas that are light grey expired before the UAV was able to cover them.  As expected, all of the paths show increasing degrees of success as the UAV velocity is increased.  The results are summarized in Table~\ref{tbl:coverage-at-100} showing the first velocity where full coverage was achieved.  These tests also illustrate the  relation between the curvature of the path and the success rate.  For a given velocity, as the curvature of the path is increased, the success rate at mapping decreases, matching with our expectations.

We also show the effect as the curvature of the path decreases, using the path modeled in Figures~\ref{fig:25-200-radius-curves}-\ref{fig:25-400-radius-curves}.  As the curvature of the path is decreased (the radius is increasing), the UAV becomes progressively more successful in mapping the path for a given UAV velocity.

\begin{table}
\caption{A comparison of velocity vs. path type, showing the minimum velocity for complete coverage.}
 \label{tbl:coverage-at-100}
\begin{center}
\begin{tabular}{@{} l  c  c @{}}
\toprule
Path & V (m/s) & \% Coverage  \\
\midrule
straight&	21	&99.96\\ 
r200&	29&	99.89\\ 
r400&	25&	99.91\\ 
r600&	25&	99.95\\ 
r800&	23&	99.96\\ 
r1000&	23&	99.94\\ 
random&	25	&99.63 \\ 
spiral& 29& 99.92 \\
\bottomrule
\end{tabular}
\end{center}
\end{table}

\begin{figure}
    \centering
    \subfloat[r200 m]{
        \includegraphics[width=0.45\linewidth]{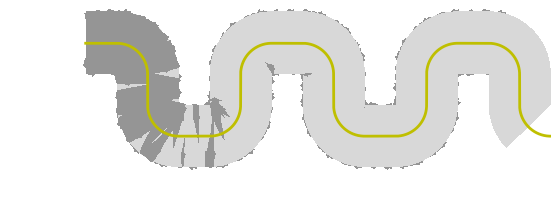}
        \label{fig:25-200-radius-curves}
    }
    \subfloat[r400 m]{
        \includegraphics[width=0.45\linewidth]{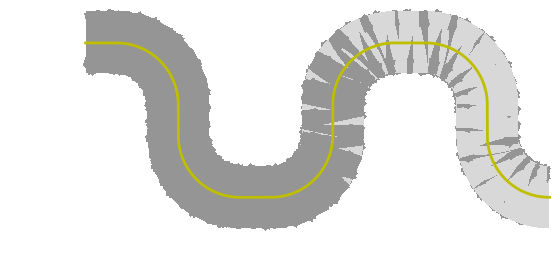}
        \label{fig:25-400-radius-curves}
    } 
    \caption{Coverage Results at $25$ m/s -- dark grey is successful coverage, light grey expired.  As the minimum path radius is increased, the UAV successfully covers a larger fraction of the total area. }
    \label{fig:fixed-speed-decreasing-curvature}
\end{figure}

\section{Conclusions and Future Work} \label{section-conclusion}

    In this paper, we defined the problem of providing path planning coverage for a moving ground vehicle.  We developed some minimum performance requirements on the part of the UAV to provide current coverage of the path immediately ahead of ground vehicle as it travels through the environment.  The plan we developed, a variation of the classic lawn mower plan, ensures complete path coverage without prior knowledge of the entire plan. We have shown that the conformal lawn mower path can be no longer than twice the length of the optimal plan, and therefore require no more than twice the UAV velocity, for any path with maximum curvature constrained less than $2/w$.  
    
    One observation from this work is that large UAV velocities are needed to successfully map a region given a reasonable ground vehicle velocity.  For our simulations, we limited the velocity of the ground vehicle to 5 m/s (or about 20 km/h), used a mapping width of 400 m, and a footprint of 100 m $\times$ 100 m.  With those parameters, the minimum required velocity for the UAV to be successful was 20 m/s, well in excess of the capability of most rotor-based UAVs, particularly for sustained flight. 
    
    Future investigations will optimize the coverage plan to the UAV operational requirements.  For paths with a high curvature, it may be appropriate to use a plan that minimizes mapping overlap given the available information. We believe that any algorithm to find the optimal plan under these constraints is likely to be NP hard~\cite{arkin2000approximation}, and this is a promising direction for future work.   We are also interested in policies for multiple UAVs working in collaboration.  From the data, we can already see that even at a slow velocity, the UAV is initially successful until the deadline/ground vehicle catches up.  Working in teams, the UAVs could coordinate their efforts to map the path and create a feasible solution.

\end{document}